\documentclass{amsart}
\usepackage[T1]{fontenc}
\usepackage{xcolor}
\usepackage{booktabs}
\usepackage{multirow}
\usepackage{subcaption}
\usepackage{graphicx}
\usepackage{hyperref}
\usepackage{url}
\usepackage{algpseudocode}
\usepackage{algorithm}
\usepackage{mathtools}
\usepackage{yhmath}
\usepackage{bbm}
\usepackage{amsthm}
\usepackage{pgfplots}
\pgfplotsset{compat=1.18}

\renewcommand{\hat}{\widehat}
\newcommand{\score}{{conformity score}}

\newtheorem{prop}{Proposition}

\newtheorem{definition}{Definition}
\definecolor{matblue}{HTML}{1F77B4}
\newtheorem{theorem}[subsection]{Theorem}

\makeatletter
\g@addto@macro{\endabstract}{\@setabstract}
\newcommand{\authorfootnotes}{\renewcommand\thefootnote{\@fnsymbol\c@footnote}}%
\makeatother

\title[Rank-Based Conformal Classification]{Trustworthy Classification through Rank-Based Conformal Prediction Sets}
\date{}

\begin{document}
\maketitle

\begin{center}
\normalsize
\authorfootnotes
Rui Luo\footnote{ruiluo@cityu.edu.hk}\textsuperscript{1} and Zhixin Zhou\footnote{zhixzhou@cityu.edu.hk}\textsuperscript{1} \par \bigskip

\textsuperscript{1}City University of Hong Kong \par\bigskip
\end{center}

\begin{abstract}
Machine learning classification tasks often benefit from predicting a set of possible labels with confidence scores to capture uncertainty. However, existing methods struggle with the high-dimensional nature of the data and the lack of well-calibrated probabilities from modern classification models. We propose a novel conformal prediction method that utilizes a rank-based score function suitable for classification models that predict the order of labels correctly, even if not well-calibrated. Our approach constructs prediction sets that achieve the desired coverage rate while managing their size. We provide a theoretical analysis of the expected size of the conformal prediction sets based on the rank distribution of the underlying classifier. Through extensive experiments, we demonstrate that our method outperforms existing techniques on various datasets, providing reliable uncertainty quantification. Our contributions include a novel conformal prediction method, theoretical analysis, and empirical evaluation. This work advances the practical deployment of machine learning systems by enabling reliable uncertainty quantification.
\end{abstract}

\section{Introduction}
Machine learning covers a wide range of classification tasks across various domains such as computer vision, natural language processing, and bioinformatics. These problems have been traditionally approached using discriminative classifiers such as logistic regression, support vector machines, and deep neural networks \cite{goodfellow2016deep}. While these methods have achieved good performance, they typically only output a single predicted class label for each input. However, in many applications, it may be useful to predict a set of possible labels along with confidence scores to capture uncertainty and allow for multiple acceptable answers.

Conformal prediction \cite{vovk2005algorithmic} provides a promising approach to address this limitation. As outlined in the comprehensive survey \cite{angelopoulos2023gentle}, conformal predictors can be combined with any underlying machine learning model to construct prediction sets or intervals guaranteed to contain the true label with a specified probability. 
Conformal prediction has found applications in various domains. In computer vision, it has been used for image classification \cite{romano2020classification,angelopoulos2021uncertainty,huang2024conformal}. In NLP, it has been used for language modeling \cite{quach2024conformal, deutschmann2024conformal}, text classification \cite{kumar2023conformal, giovannotti2021transformer}, and question answering \cite{fisch2021efficient}. Conformal prediction methods have also been used in time series forecasting \cite{xu2023conformal, auer2024conformal} and astronomy \cite{ashton2024calibrating}. In graph-based tasks, conformal prediction has been applied to node prediction \cite{huang2024uncertainty, lunde2023conformal}, edge prediction \cite{luo2023anomalous, zhao2024conformalized, marandon2024conformal, luo2024conformal}, and graph encoding \cite{severo2023one}.

Conformal classification has been extended to various settings and applications. For multi-label classification, methods have been developed to handle instances with multiple labels \cite{cauchois2021knowing, tyagi2023multi}. Mondrian conformal prediction \cite{vovk2005algorithmic} focuses on achieving category-wise validity. Approaches for non-exchangeable data \cite{tibshirani2019conformal,podkopaev2021distribution} and classification under ambiguous ground truth labels \cite{stutz2022learning} have also been proposed. In binary classification, false discovery rate control has been explored \cite{bates2023testing}. Conformal risk control has been applied to multi-label classification \cite{angelopoulos2024conformal}.

Despite the progress, several challenges persist in reliable uncertainty quantification in machine learning. One major hurdle is scaling conformal prediction methods to large datasets and complex, potentially mis-specified models. Poor calibration \cite{guo2017calibration}, overfitting \cite{recht2019imagenet}, bias \cite{mehrabi2021survey}, and performance degradation under distribution shifts \cite{koh2021wilds} pose significant challenges in these settings. Overcoming these obstacles will greatly enhance the practicality and impact of conformal prediction in real-world applications, enabling more reliable and robust machine learning systems. 

In this work, we propose a novel conformal classification method that introduces a rank-based conformity score function (RANK) to address challenges in applying conformal prediction to models with poorly calibrated probabilities. Our approach predicts sets of possible labels for uncertain instances while managing the size of the prediction sets. Unlike existing adaptive prediction set algorithms, such as APS \cite{romano2020classification}, our method does not rely on strong assumptions about the model's probabilities. Instead, it constructs prediction sets using a two-step threshold based on the ranking and probability of the labels, achieving desired coverage while controlling the set size.
In Section~\ref{sec:related:work}, we provide a comprehensive overview of existing methods. This is followed by a detailed description of our proposed approach in Section~\ref{sec:ours}, and a theoretical analysis in Section~\ref{sec:theory}.

Through extensive experiments on benchmark datasets from various domains, we demonstrate the validity and efficiency of our proposed method. The results show that our approach substantially outperforms existing conformal classification methods, achieving the specified coverage level with prediction sets that are often much smaller. 
Our key contributions are:
\begin{enumerate}
\item A novel conformal prediction method based on a rank-based conformity score function, suitable for classification models that may not output well-calibrated probabilities.
\item Theoretical analysis of the coverage guarantee and the expected size of the conformal prediction sets based on the rank distribution of the underlying classifier.
\item Extensive empirical evaluation demonstrating the effectiveness of our method in providing reliable uncertainty quantification for classification tasks across various domains.
\end{enumerate}

By enabling modern classification models to output prediction sets with rigorous uncertainty quantification, our work takes an important step towards the reliable deployment of machine learning systems in practical applications. More broadly, we advance the theory and practice of conformal prediction for classification tasks where well-calibrated probabilities may not be available.

\section{Related Work}
\label{sec:related:work}

Many studies have focused on the following framework to determine the optimal prediction set for conformal prediction in classification tasks.

The process consists of two steps. First, a predictive model $\hat{\pi}$ is trained on the training set to estimate the probability or score of an input belonging to each of the $K$ possible categories. In the second step, for a fixed $\alpha$, a prediction set with $1-\alpha$ coverage is constructed for each test input using $\hat{\pi}(x_i)$. The challenge lies in determining which class labels to include in the prediction set based on the conformity scores, which are typically derived from the calibration set. The expected size of the prediction set depends on the construction method, and different methods are compared based on this metric.

Given this approach to guarantee the $1-\alpha$ coverage of the prediction set, researchers have focused on how to define the prediction set. There are two main ideas, which can be described as the following two methods. These methods do not rely on the assumption that $\hat{\pi}(x_i)$ represents well-calibrated probabilities, but understanding $\hat{\pi}$ in this way can make their ideas easier to grasp.

The first method is to determine an individual threshold. Threshold Conformal Prediction (THR)~\cite{sadinle2019least} is a representative of this method, which includes labels with sufficiently large $\hat \pi(x_i)$ values in the prediction set. The conformal score is $1-\hat\pi_{y}(x)$. In the calibration set, a threshold $t$ is determined such that ${k: \hat \pi_k(x_i)\ge t}$ forms the prediction set and guarantees $1-\alpha$ coverage. This $t$ is roughly the $100(1-\alpha)$-th percentile of $\hat\pi_{y_i}(x_i)$ for $(x_i, y_i)$ in the calibration set. If $\hat \pi(x)$ perfectly matches the true posterior probabilities, THR achieves the smallest expected size of the prediction set. However, other more robust methods have been developed to account for the discrepancy between estimated and true probabilities.

The second method is to determine a cumulative threshold. Adaptive Prediction Set (APS)~\cite{romano2020classification} is a representative of this method. First, the outcomes of $\hat \pi_k(x_i)$ for $k=1,\dots,K$ are sorted. In the calibration set, $\sum_{k=y_i}^K \hat \pi_{(k)}(x_i)$ is treated as the $p$-value of each sample. The $p$-value that should be tolerated to guarantee $1-\alpha$ coverage is calculated using the calibration set. Given the $p$-value threshold $t$, when $\hat \pi(x_i)$ is observed in the test set, $k'$ is found such that $\sum_{k=k'}^{K}\hat \pi_{(k)}(x_i) \approx t$. The prediction set is then given by the labels corresponding to the top $k'$ values of $\hat \pi_k(x_i)$. In this method, $t$ is a value of the cumulative sum of $\hat \pi$, and changing $t$ corresponds to varying the cumulative sum of $\hat \pi_{(k)}(x_i)$. Later works, such as Regularized APS (RAPS) \cite{angelopoulos2021uncertainty} and Sorted APS (SAPS) \cite{huang2024conformal}, extend this type of method.

While both THR and APS have been shown to achieve valid coverage, they have some limitations. THR assumes that the individual $\hat{\pi}(x_i)$ values, particularly the top ones, are well-calibrated, an assumption that may not be valid for modern deep learning models. APS, on the other hand, requires an accurate estimate of the tail area of the predictive probability distribution $\hat{\pi}(x_i)$, which may not fully capture the model's uncertainty. This can result in unacceptably large prediction sets, especially for datasets with a large number of classes and very small $\alpha$ values.

In this work, we propose a novel conformal prediction method that addresses these limitations. Our approach is based on a rank-based conformity score function that is suitable for classification models that may not output well-calibrated probabilities but can rank the labels correctly. By exploiting the rank information instead of the raw $\hat{\pi}(x_i)$ values, our method is more robust to miscalibration and can efficiently construct prediction sets with rigorous coverage guarantees.

\section{Proposed Method}
\label{sec:ours}

\subsection{Problem Setup}
We consider a general $K$-class classification problem with a dataset $\{(x_i, y_i)\}_{i=1}^n$, where $x_i \in \mathcal{X}$ is the input feature vector and $y_i \in \mathcal{Y} = \{1, \dots, K\}$ is the corresponding class label. The dataset is split into three parts: a training set $\mathcal{D}_{\text{train}}$, a calibration set $\mathcal{D}_{\text{cal}}$, and a test set $\mathcal{D}_{\text{test}}$. We assume that the examples in these sets are exchangeable.

Our goal is to construct a conformal predictor that outputs a prediction set $\mathcal{C}(x) \subseteq \mathcal{Y}$ for each test input $x$ such that the true label $y$ is included in $\mathcal{C}(x)$ with a probability of at least $1-\alpha$, where $\alpha \in (0, 1)$ is a user-specified significance level. Formally, we aim to achieve the following coverage guarantee:
\begin{equation}
\mathbb{P}(y \in \mathcal{C}(x)) \geq 1-\alpha
\end{equation}

\subsection{Rank-based Conformity Score}
Our approach directly focuses on the goal of minimizing the size of the prediction set. In the calibration set, we evaluate the size of the prediction set required to include the true label. We operate under the assumption that a higher value of $\hat \pi_k(x_i)$ indicates a greater likelihood of $x_i$ belonging to class $k$. Consequently, we impose a constraint on the confidence interval: if class $k$ is included in the prediction set, then any class $k'$ such that $\hat\pi_{k'}(x_i)>\hat\pi_{k}(x_i)$ must also be included in the prediction set.

Given this constraint, the smallest prediction set that includes the true label will be determined by the rank of $\hat \pi_{y_i}(x_i)$ within the sequence $\{\hat\pi_1(x_i), \dots, \hat\pi_K(x_i)\}$. However, it is common to encounter multiple samples in the calibration set that have the same prediction set size. In such cases, we need to establish a preference for breaking ties.

\begin{algorithm}[!ht]
\caption{Rank-based conformal classification}
\begin{algorithmic}[1]
\State \textbf{Input:} data $\{(x_i, y_i)\}_{i\in\mathcal I}$, a test sample $x_{n+1}$,  black-box learning algorithm $\mathcal{B}$, level $\alpha \in (0, 1)$.
\State Randomly split the indices $\mathcal I$ into two subsets $\mathcal I_1, \mathcal I_2$.
\State Train $\mathcal{B}$ on all samples in $\mathcal I_1: \hat\pi \leftarrow \mathcal{B}(\{(x_i, y_i): i \in \mathcal I_1\})$.
\State $n\leftarrow |\mathcal I_2|$. Find the $\lfloor(n+1)\alpha\rfloor$-th largest value in (\ref{eq:r}), denoted by $r_\alpha^*$. \label{step:four}
\State Find the proportion $p$ in (\ref{eq:p}). 
\State Find the $\lceil np\rceil$-th largest value in (\ref{eq:hat:pi}), denoted by $\pi^*$. \label{step:six}
\State With $r^*_\alpha$ and $\pi^*$ obtained from Step~\ref{step:four} and Step~\ref{step:six}, use the function $\hat C(x_{n+1})$ in  \eqref{eq:conf:int} to construct the prediction set for $x_{n+1}$.
\State \textbf{Output:} A prediction set $\hat C_\alpha(x_{n+1})$.
\end{algorithmic}\label{alg:rank}
\end{algorithm}

Intuitively, our tie-breaking approach aims to efficiently cover the true label by favoring the option with the larger predicted probability. When comparing the $k$th most likely label of $x_i$ and $x_j$, given the predicted probabilities $\hat \pi$, we prioritize the inclusion of labels that are more confidently predicted by the model. This choice aligns with our goal of constructing prediction sets that are more likely to contain the true label while maintaining a smaller overall size compared to randomly breaking ties.

We will rigorously summarize the idea above. Our goal is to determine a rank $k$ such that for the test sample, we include either the top $k$ or $k-1$ labels in the prediction set. We also need to establish a rule to choose between $k$ and $k-1$. These rules will be determined using the calibration set. The method to determine $k$ is straightforward. Let $\mathcal I_2$ be the calibration set and $n=|\mathcal I_2|$. For each $i\in \mathcal I_2$, we define the following rank:
\begin{align}\label{eq:r}
    r_i = \text{rank of } \pi_{y_i}(x_i) \text{ in } \{\pi_k(x_i):k\in[K]\}. 
\end{align}
We then find the order statistics of these ranks: $r_{(1)}\ge r_{(2)}\ge\dots\ge r_{(n)}$ and let $r^*_\alpha=r_{(\lfloor (n+1)\alpha\rfloor)}$. This ensures that:
\begin{align*}
    |\{i\in\mathcal I_2: r_i\le r^*_\alpha-1\}|<\lfloor(n+1)\alpha\rfloor \le|\{i\in\mathcal I_2: r_i\le r^*_\alpha\}|.
\end{align*}
To construct the prediction set for $x_{n+1}$, we will include either the top-$(r^*_\alpha-1)$ or top-$r^*_\alpha$ classes based on the values of $\pi_1(x_{n+1}), \dots, \pi_K(x_{n+1})$. The top-$r^*_\alpha$ classes refer to the classes corresponding to the $r^*_\alpha$ largest values among $\pi_1(x_{n+1}), \dots, \pi_K(x_{n+1})$.  To achieve $1-\alpha$ coverage, we need to determine when to include the $r^*_\alpha$-th class and when not to. We start by calculating the proportion $p$ of instances for which we should include the $r^*_\alpha$-th label:
\begin{align}\label{eq:p}
    p:=\frac{n-\lfloor(n+1)\alpha\rfloor-|\{i\in\mathcal I_2: r_i\le r^*_\alpha-1\}|}{|\{i\in\mathcal I_2: r_i = r^*_\alpha\}|}.
\end{align}
Roughly speaking, the numerator of this fraction represents the difference between the number of samples obtained by selecting all samples with rank $r_i \leq r^* - 1$ and the number of samples needed to achieve a coverage of $n(1-\alpha)$ in the calibration set. Next, we find the $\lceil np\rceil$-th largest value, denoted as $\pi^*$, in the set of $\pi_{(r^*_\alpha)}(x_i)$'s:
\begin{align}\label{eq:hat:pi}
    \pi^*=\lceil np\rceil\text{-th largest value in} \{\hat \pi_{(r^*_\alpha)}(x_i): i\in\mathcal I_2\},
\end{align}
where $\hat\pi_{(k)}(x_i)$ denotes the $k$-th order statistics in $(\pi_1(x_i), \dots, \pi_n(x_i))$. 
Finally, for the test sample $x_{n+1}$, if $\hat \pi_{(r^*_\alpha)}(x_{n+1})\ge \pi^*$, then the $r^*_\alpha$-th label will be included in the prediction set. Otherwise, it will not be included. To summarize rigorously, for a test sample $x_{n+1}$, 
\begin{align}\label{eq:conf:int}
    \hat C_\alpha(x_{n+1}) = 
    \begin{cases}
        \{k:\hat\pi_k(x_{n+1})\ge \hat\pi_{( r^*_\alpha)}(x_{n+1})\}, \\
        \text{ if } \hat\pi_{(r^*_\alpha)}(x_{n+1})\ge \pi^*; \\
        \{k:\hat\pi_k(x_{n+1})\ge \hat\pi_{(r^*_\alpha-1)}(x_{n+1})\}, \\
        \text{otherwise. }
    \end{cases}
\end{align}
This definition implies the following proposition.
\begin{prop}
    The output $\hat C_\alpha(x_{n+1})$ from Algorithm~\ref{alg:rank}  satisfies $\hat C_\alpha(x_{n+1})\subset\{\hat y_{(1)}, \dots, \hat y_{(r^*_\alpha)}\}$, i.e., the subset of labels that have top-$r^*_\alpha$ values in $\{\hat\pi_1(x_{n+1}), \dots, \hat\pi_K(x_{n+1})\}$.
\end{prop}
We note that we can define the following conformity score for our method. This will help us understand why $1-\alpha$ coverage is guaranteed. 
\begin{align*}%
    \begin{split}
    c_i&=c(x_i, y_i)\\
    &= [\text{rank of } \hat\pi_{y_i}(x_i) \text{ in } \{\hat\pi_1(x_i), \dots, \hat \pi_K(x_i)\}] - 1\\
    &+\frac 1n[\text{rank of } \hat \pi_{y_i}(x_i)\text{ in } \{\hat\pi_{y_i}(x_1), \dots, \hat \pi_{y_i}(x_n)\} ]. 
    \end{split}
\end{align*}
Let's define the quantile $\hat Q_{\alpha}$ as the $\lfloor (n+1)\alpha\rfloor$-th largest value among the conformity scores ${c_1, c_2, \dots, c_{n}}$ calculated on the calibration set. To construct a prediction set with $1-\alpha$ coverage, we include all samples from the calibration set whose conformity scores $c_i$ are less than or equal to the quantile $\hat Q_\alpha$. In other words, the procedure for defining the prediction set is equivalent to selecting the calibration samples that satisfy the condition $c_i\le \hat Q_\alpha$, which ensures the desired coverage level of $1-\alpha$.

\begin{figure}[!htb]
\begin{center}
\begin{tikzpicture}[scale=0.7]
\begin{axis}[
    ybar,
    bar width=0.7,
    bar shift=0,
    xlabel={rank of the true class},
    ylabel={frequency},
    xmin=0.1,
    xmax=10.3,
    xtick={1,2,3,4,5,6,7,8,9,10},
    xticklabel style={/pgf/number format/fixed},
    ymin=0,
    ymax=0.68,
    ytick={0, 0.1, 0.2, 0.3, 0.4, 0.5, 0.6},
    yticklabel style={/pgf/number format/fixed},
    nodes near coords,
    every node near coord/.append style={font=\small, rotate=90, anchor=west},
    nodes near coords align={vertical},
    clip=false
]
\addplot[ybar,fill=matblue, draw=none, /pgf/number format/precision=5] coordinates {(1, 0.6) (2, 0.26) (3, 0.06) (4, 0.04) (5, 0.03) (6, 0.01) (7, 0) (8, 0) (9, 0) (10, 0)};
\draw[line width=1, dashed, color=red] (axis cs:2.5,0) -- (axis cs:2.5,0.6);,
\end{axis}
\end{tikzpicture}
\qquad %
\begin{tikzpicture}[scale=0.7]
\begin{axis}[
    ybar,
    bar width=0.03,
    bar shift=0,
    xlabel={3rd largest probability $\hat{\pi}_{(3)}$},
    ylabel={frequency},
    xmin=-0.02,
    xmax=0.35,
    xtick={0,.05,.1,.15,.2,.25,.3},
    xticklabel style={/pgf/number format/fixed},
    ymin=0,
    ymax=0.48,
    ytick={0, 0.1, 0.2, 0.3, 0.4, 0.5},
    yticklabel style={/pgf/number format/fixed},
    nodes near coords,
    every node near coord/.append style={font=\small, rotate=90, anchor=west},
    nodes near coords align={vertical},
    clip=false
]
\addplot[ybar interval,fill=matblue, draw=none, /pgf/number format/precision=5] coordinates {(0, 0.4) (0.05, 0.35)  (0.1, 0.1) (0.15, 0.1) (0.2, 0.05) (0.25, 0)};
\draw[line width=1, dashed, color=red] (axis cs:0.15,0) -- (axis cs:0.15,0.42);,
\end{axis}
\end{tikzpicture}
\end{center}
\caption{The figure illustrates the construction of a 90\% prediction set for a test sample with sorted probability vector $[0.55, 0.2, \textcolor{red}{0.15}, 0.1, 0, 0, 0, 0, 0, 0]$. The top three classes are included based on both the probability of ranks (\textbf{Left}) and the distribution of the 3rd largest probabilities in the calibration set (\textbf{Right}).}
\label{fig:example}
\end{figure}
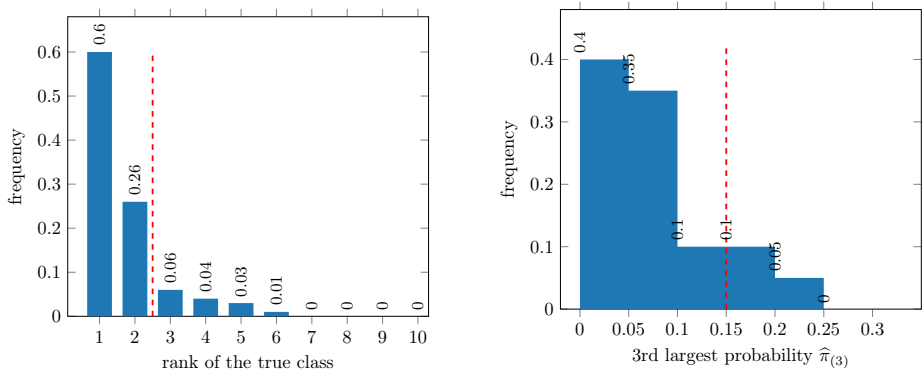

\paragraph{Example:} Let us consider the following example to understand our method better. After applying a training algorithm to the training samples with $K=10$ classes, we obtain the function $\hat{\pi}$. By applying $\hat{\pi}$ to the calibration samples $\{(x_i, y_i)\}_{i \in \mathcal{I}_2}$, we obtain the ranks of $\hat{\pi}_{y_i}(x_i)$ for $i \in \mathcal{I}_2$, which represent the ranks of the true class. This forms an empirical distribution on the set $\{1, 2, \dots, 10\}. $ Let us consider an example in 
Figure \ref{fig:example}. $(0.6, 0.26, 0.06, 0.04, 0.03, 0.01, 0, 0, 0, 0)$ is the empirical distribution mentioned above. Suppose we want to construct the prediction set $\hat C_\alpha$ with $\alpha=0.1$, i.e., coverage probability equals to 0.9. The rank distribution in the empirical distribution shows that the top two classes have a cumulative probability of 0.86, which falls short of the desired coverage of $0.9$. Including the top three classes increases the cumulative probability to 0.92, exceeding the desired coverage. Therefore, the size of the prediction set in this case will be 2 or 3. It takes value 2 or 3 depending on the result of $\hat \pi$ applying on the test sample. Applying $\hat\pi$ on a test sample $x_{n+1}$ and obtain a sorted probability vector $(\hat{\pi}_{(1)}(x_{n+1}), \hat{\pi}_{(2)}(x_{n+1}), \dots, \hat{\pi}_{(10)}(x_{n+1}))=(0.55, 0.2, 0.15, 0.1, 0, 0, 0, 0, 0, 0)$.  

To determine whether the class with rank 3 should be included in the prediction set, we compare $\hat{\pi}_{(3)}(x_i)$ to the distribution of the 3rd largest probability. The calculation $\frac{(1-\alpha) - 0.86}{0.92-0.86} = \frac{2}{3} < 0.85$ indicates that the class with rank 3 should be included in the prediction set. Another equivalent way to think of this would be the probability 0.15 is the top 15\% among all $\{\hat \pi_{(3)}(x_i): i=1, 2, \dots, n\}$, so the rank 3 class should be included in the prediction set $\hat C_\alpha$. 

\subsection{Comparison with existing works} 
Our method may seem different from the approaches in Section~\ref{sec:related:work}, but there are connections. If $\hat\pi_1(x_i), \dots, \hat\pi_K(x_i)$ are nearly identical, the ascending order of ranks in equation~\eqref{eq:r} is almost equivalent to the descending order of $p$-values in APS, suggesting our method makes fewer assumptions about $\hat\pi$. In tie-breaking, APS uses a uniform random variable, while we use the THR idea to include labels with sufficiently large $\hat\pi$ values. Thus, our method incorporates aspects of both APS and THR.

\section{Theoretical Coverage Guarantee}\label{sec:theory}

In this section, we will demonstrate that our approach can theoretically achieve $1-\alpha$ coverage. To begin, we will define the concept of exchangeability of random variables. 

\begin{definition}[Exchangeability]
Let $Z_1, Z_2, \dots,$ $Z_n$ be a sequence of random variables. The sequence is said to be exchangeable if, for any permutation $\pi$ of the indices $[n]$, the joint distribution of the permuted sequence $(Z_{\pi(1)}, Z_{\pi(2)}, \dots, Z_{\pi(n)})$ is identical to the joint distribution of the original sequence $(Z_1, Z_2, \dots, Z_n)$. 
\end{definition}

This assumption about the dataset is widely used when considering calibration samples and test samples~\cite{romano2020classification,huang2024conformal}.  Assuming exchangeability, we can demonstrate the following result: for a test sample $X_{n+1}$ that has not been seen in the training or calibration set, the prediction set output by Algorithm~\ref{alg:rank} will include the true label $Y_{n+1}$ with a probability of at least $1-\alpha$.

\begin{theorem}\label{thm: coverage}
If the samples $(X_i, Y_i)$, for $i \in [n+1]$, are exchangeable and $\mathcal{B}$ from Algorithm~\ref{alg:rank} is invariant to permutations of its input samples, the output of Algorithm~\ref{alg:rank} satisfies:
\begin{equation}
\mathbb{P}\left(Y_{n+1} \in \hat{C}_{\alpha}(X_{n+1})\right) \geq 1 - \alpha.
\end{equation}
\end{theorem}

\begin{proof}
By the definition of the prediction set $\hat C_\alpha$, $Y_{n+1}\in \hat C_\alpha(X_{n+1})$ is equivalent to the \score{} $c_{n+1} = c(x_{n+1}, y_{n+1})< c_{(\lfloor(n+1)\alpha\rfloor)}$. 
By the exchangeablility of $(X_1, Y_1), \dots, (X_{n+1}, Y_{n+1})$, $c(X_1, Y_1), \dots, c(X_{n+1}, Y_{n+1})$ is also exchangeable. The rank of $c(X_{n+1}, Y_{n+1})$ has a uniform distribution on $\{1, \dots, n+1\}$. The rank of $c(X_{n+1}, Y_{n+1})$ greater than $\lfloor (n+1)\alpha\rfloor$ with probability $\frac{n+1-\lfloor (n+1)\alpha\rfloor}{n+1}$, which is at least $1-\alpha$.
\end{proof}

\section{Experiments}
This section presents experiments that evaluate the performance of prediction sets generated by various methods, including APS \cite{romano2020classification}, RAPS \cite{angelopoulos2021uncertainty}, SAPS \cite{huang2024conformal}, and our proposed method (RANK), on CV and NLP tasks involving multiclass classification.

\begin{figure*}[!htb]
    \centering
    
    \begin{subfigure}[b]{0.23\textwidth}
        \includegraphics[width=\textwidth]{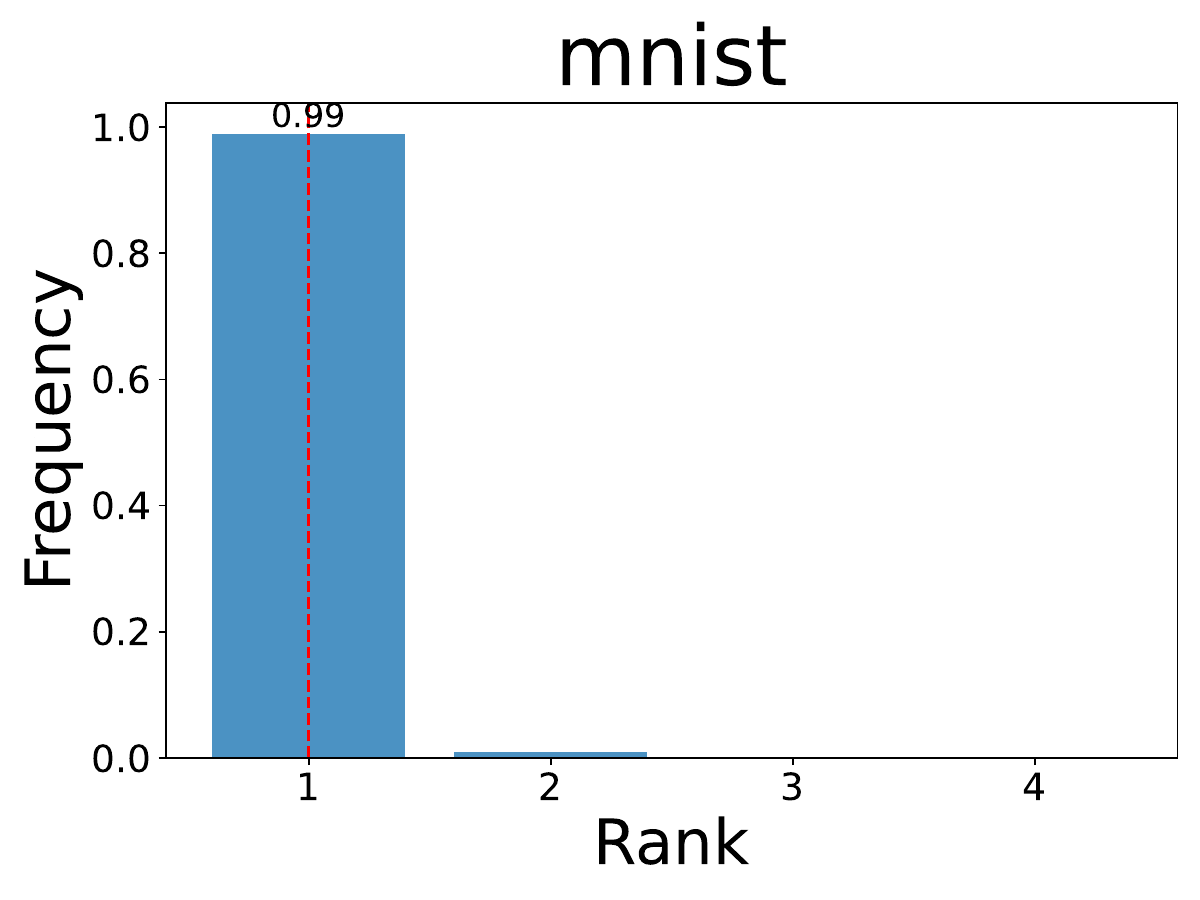}
    \end{subfigure}
    \hspace{0.006\textwidth}
    \begin{subfigure}[b]{0.23\textwidth}
        \includegraphics[width=\textwidth]{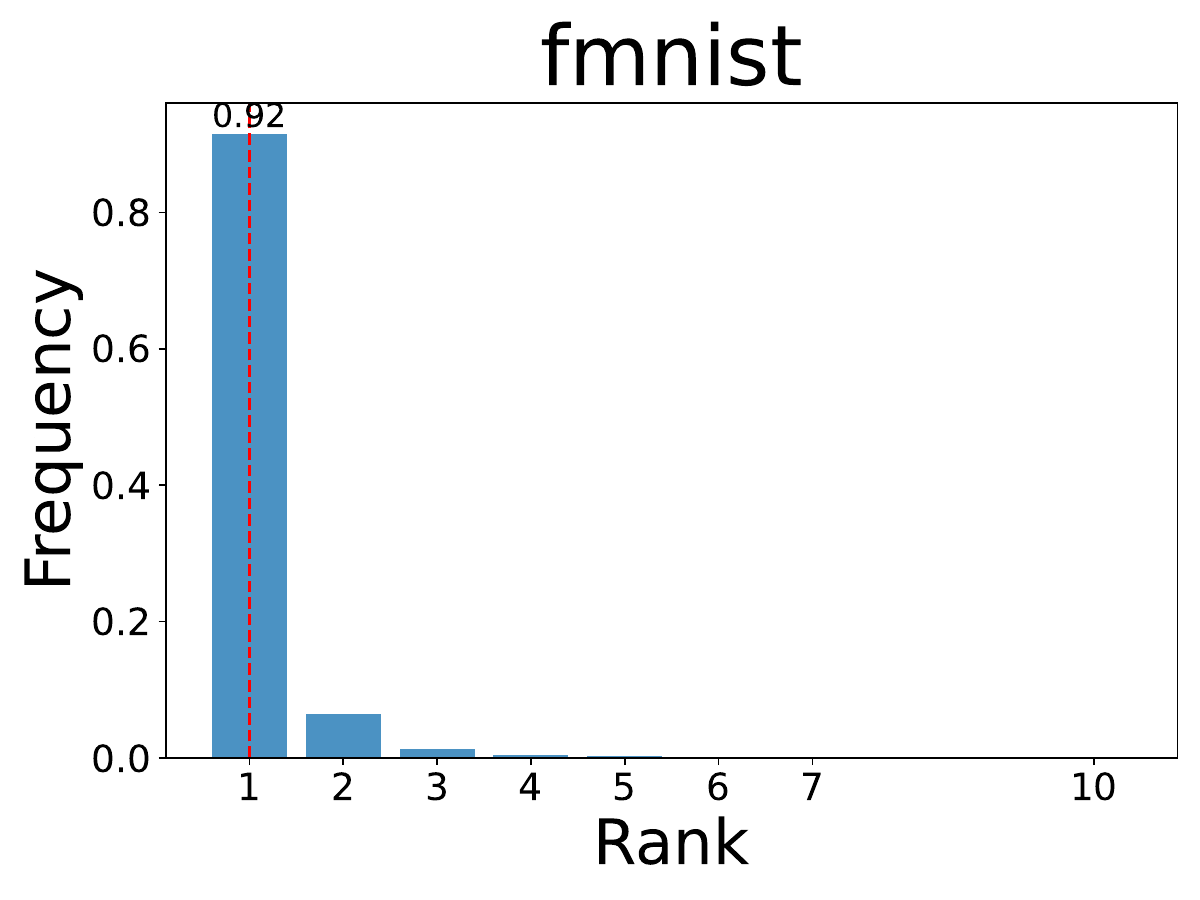}
    \end{subfigure}
    \hspace{0.006\textwidth}
    \begin{subfigure}[b]{0.23\textwidth}
        \includegraphics[width=\textwidth]{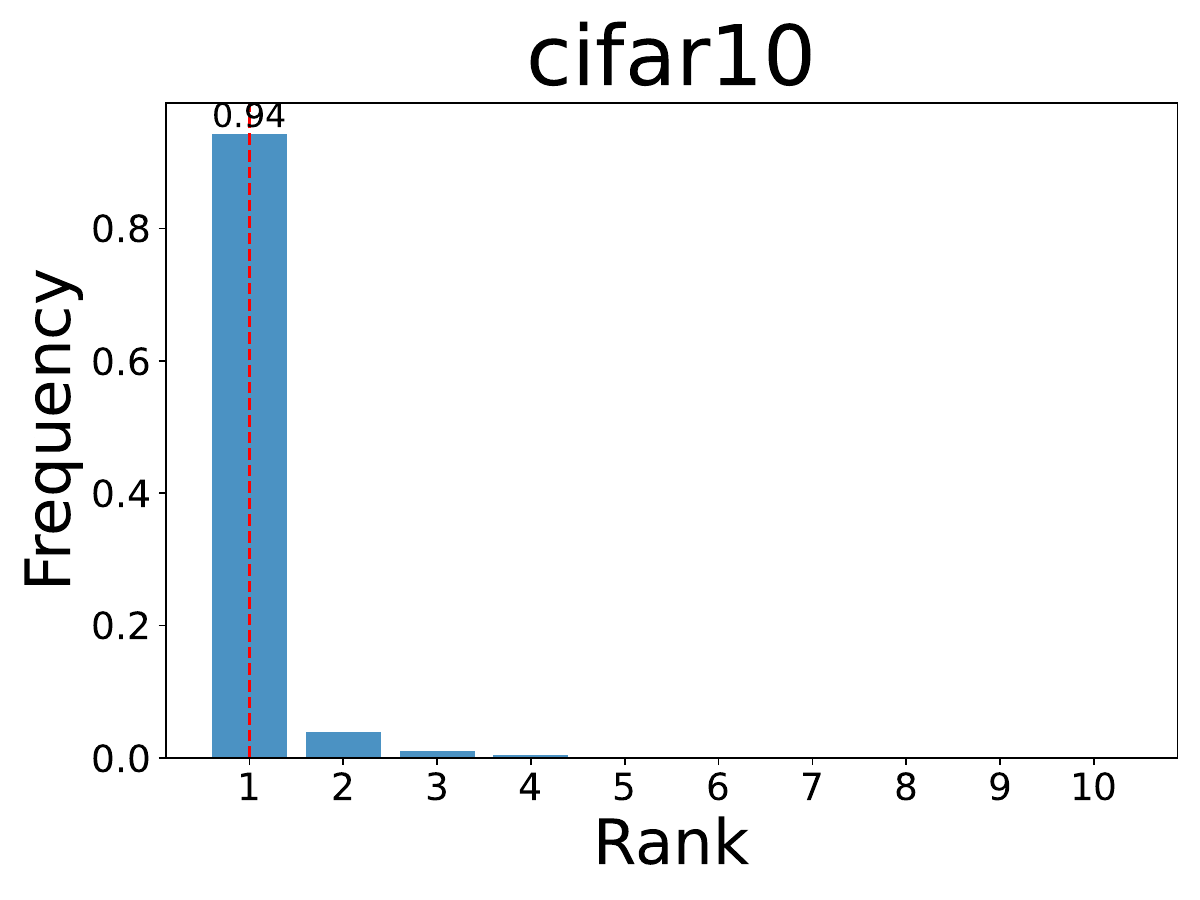}
    \end{subfigure}
    \hspace{0.006\textwidth}
    \begin{subfigure}[b]{0.23\textwidth}
        \centering
        \includegraphics[width=\textwidth]{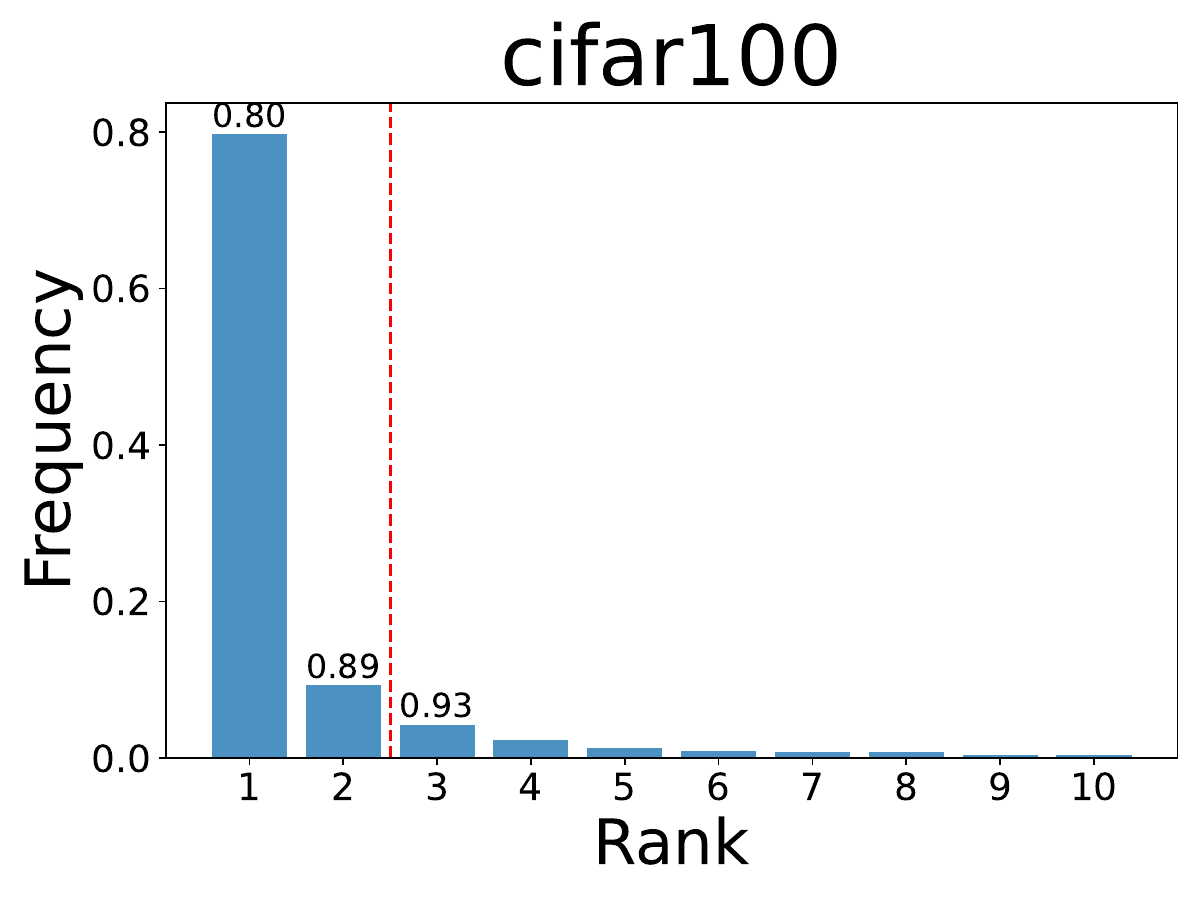}
    \end{subfigure}
    
    \begin{subfigure}[b]{0.23\textwidth}
        \includegraphics[width=\textwidth]{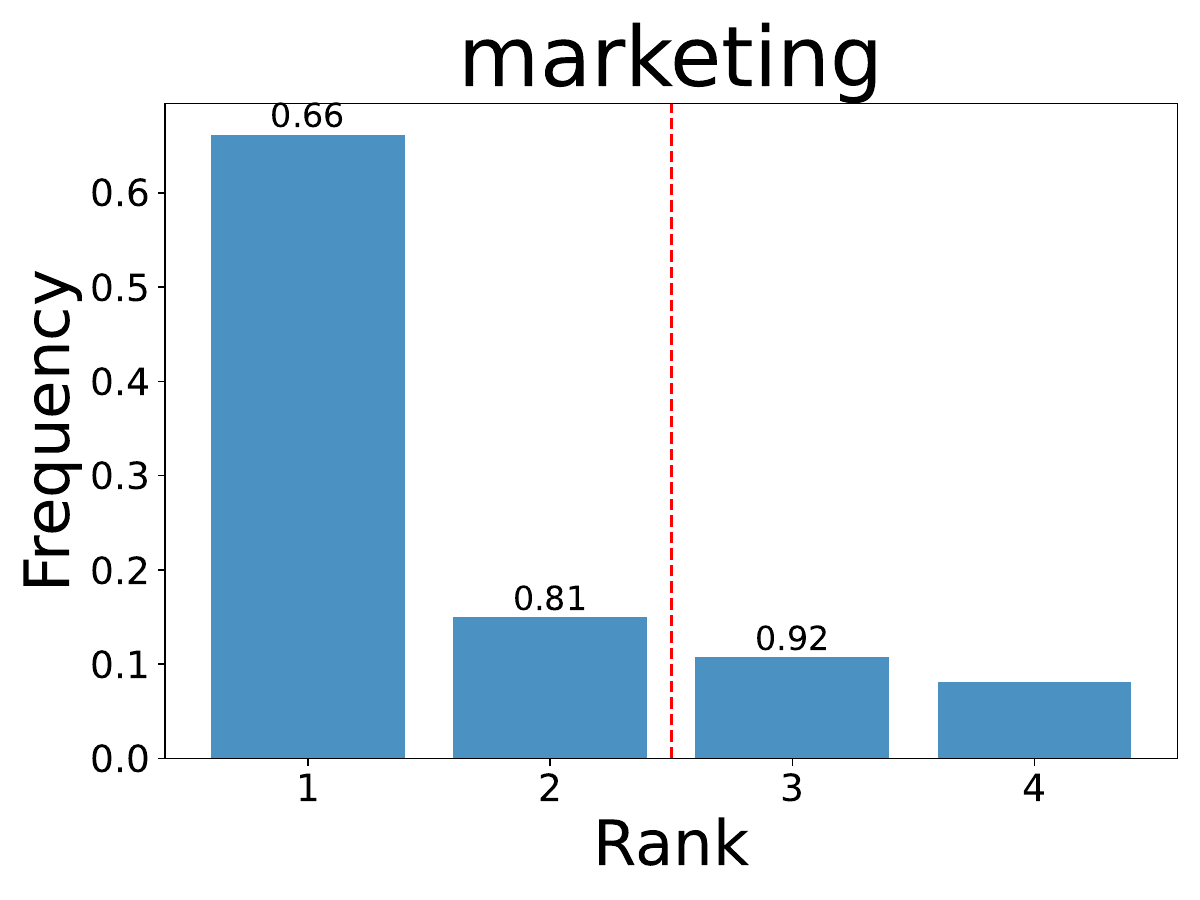}
    \end{subfigure}
    \hspace{0.01\textwidth}
    \begin{subfigure}[b]{0.23\textwidth}
        \includegraphics[width=\textwidth]{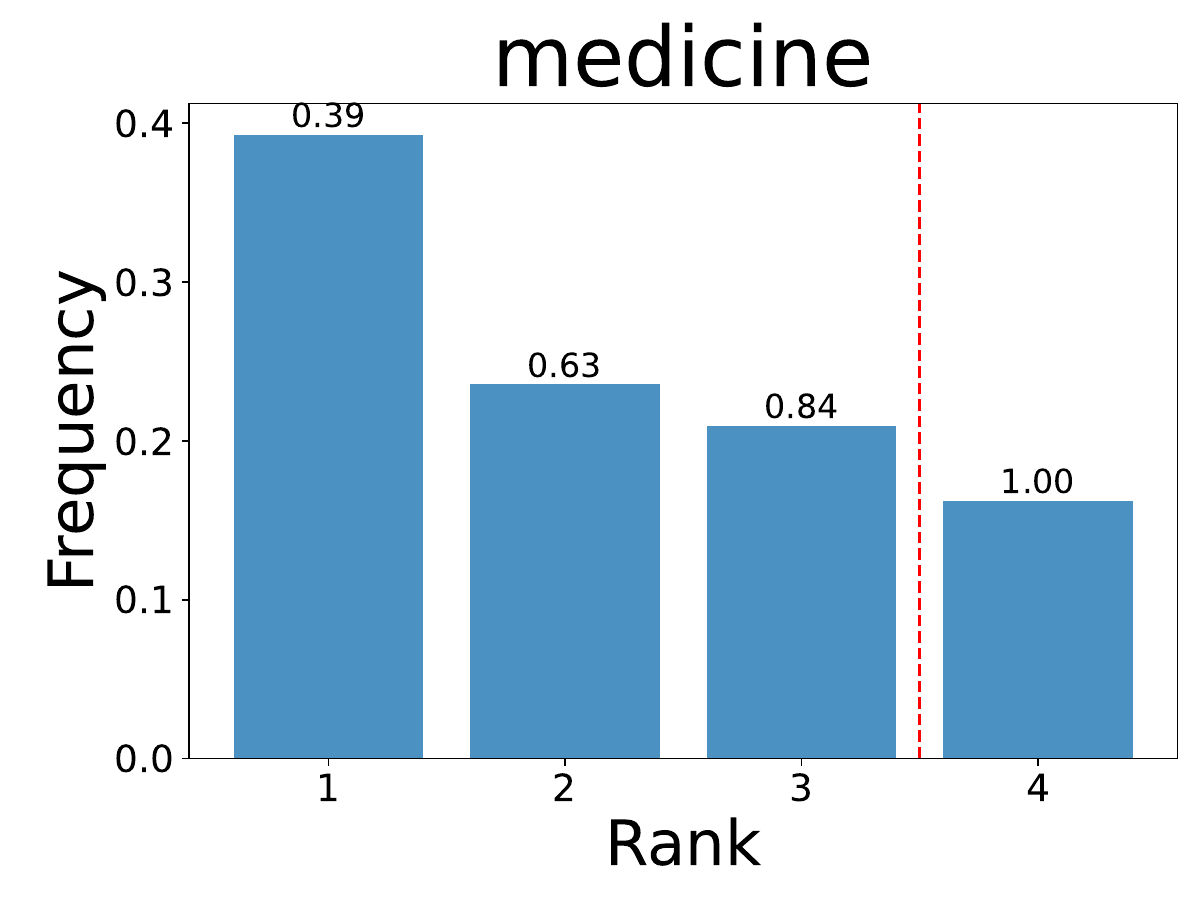}
    \end{subfigure}
    \hspace{0.01\textwidth}
    \begin{subfigure}[b]{0.23\textwidth}
        \includegraphics[width=\textwidth]{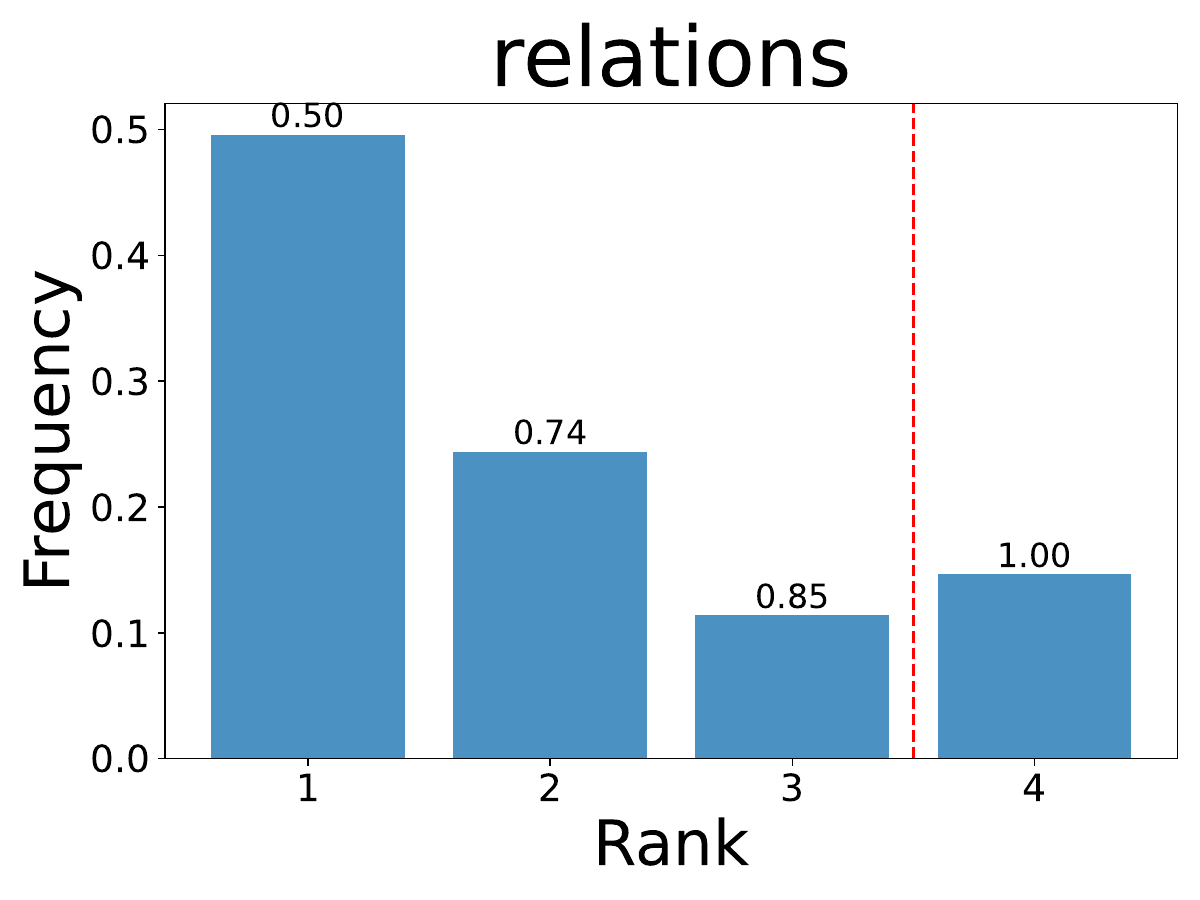}
    \end{subfigure}

    \begin{subfigure}[b]{0.23\textwidth}
        \includegraphics[width=\textwidth]{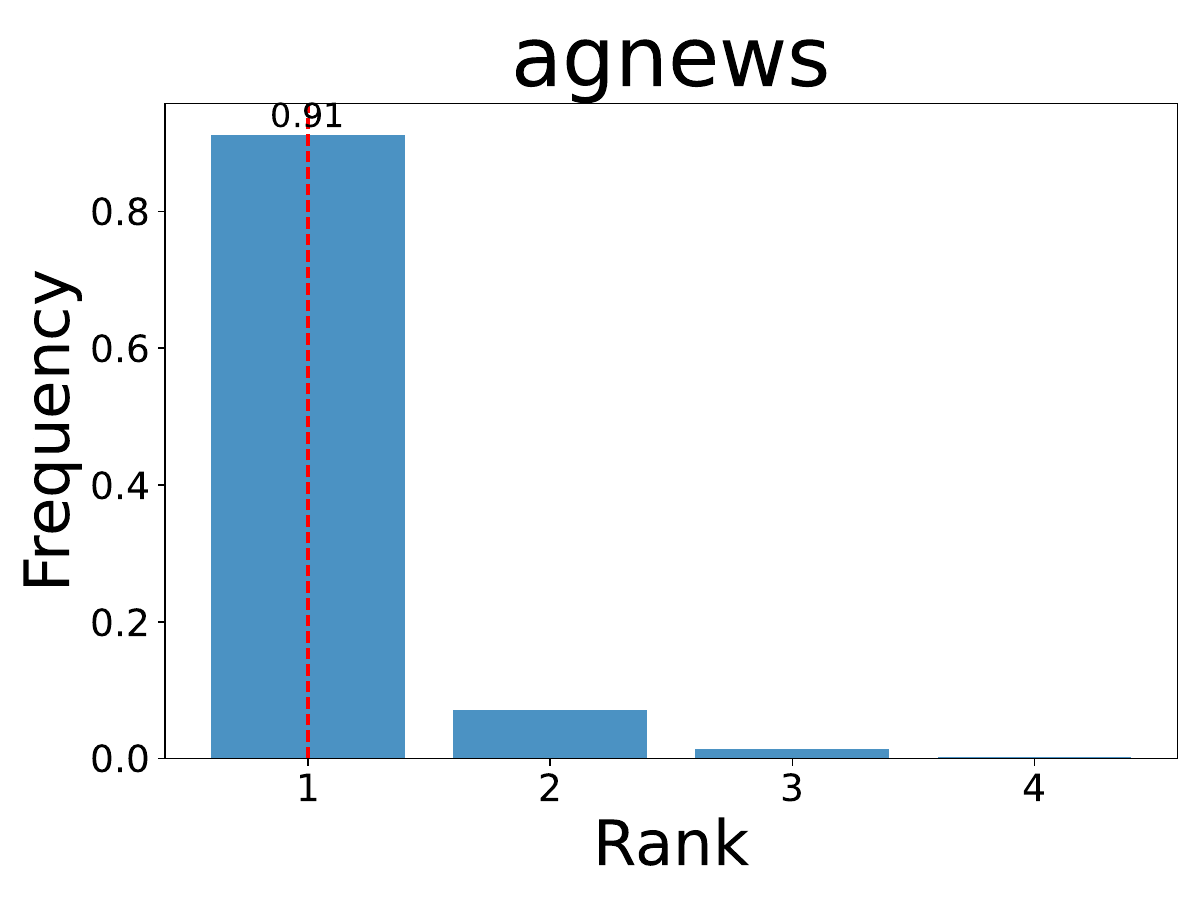}
    \end{subfigure}
    \hspace{0.006\textwidth}
    \begin{subfigure}[b]{0.23\textwidth}
        \includegraphics[width=\textwidth]{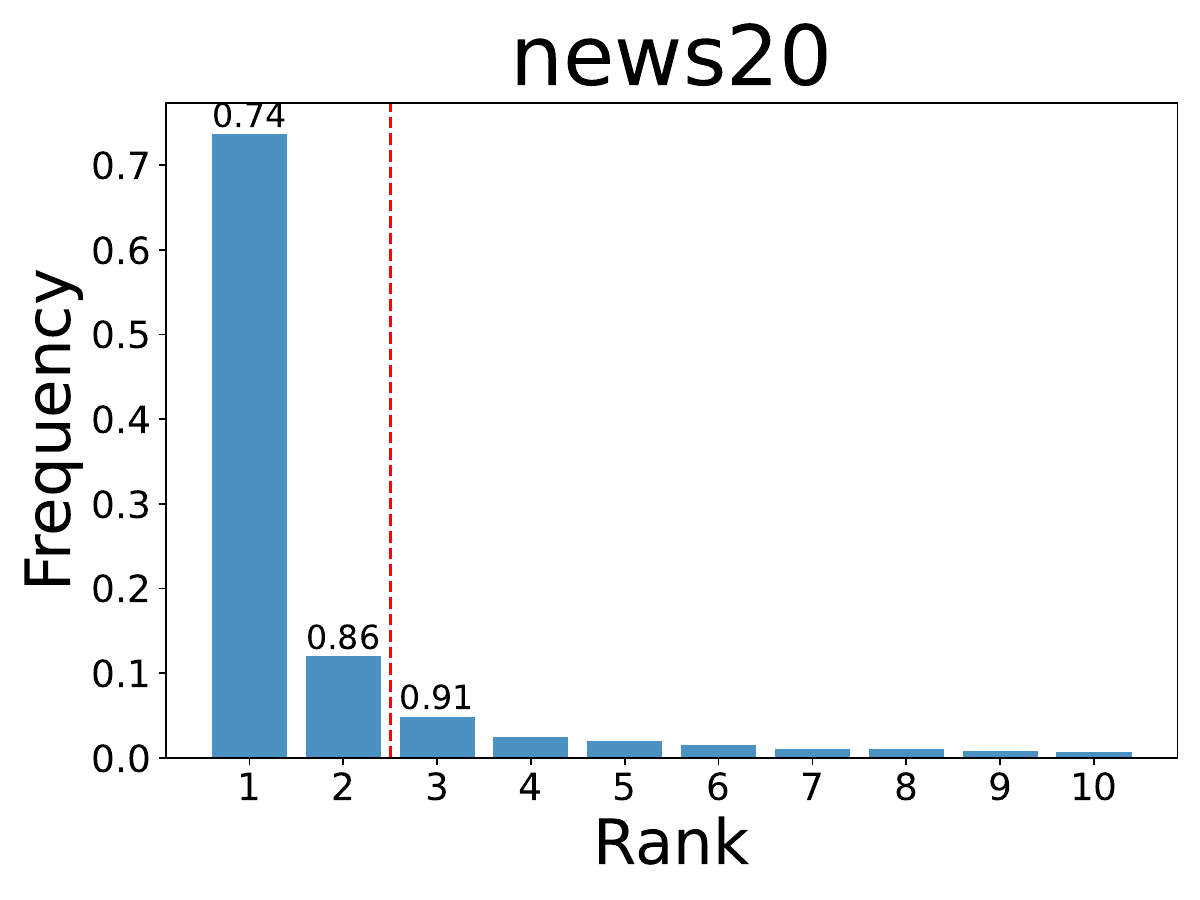}
    \end{subfigure}
    \hspace{0.006\textwidth}
    \begin{subfigure}[b]{0.23\textwidth}
        \includegraphics[width=\textwidth]{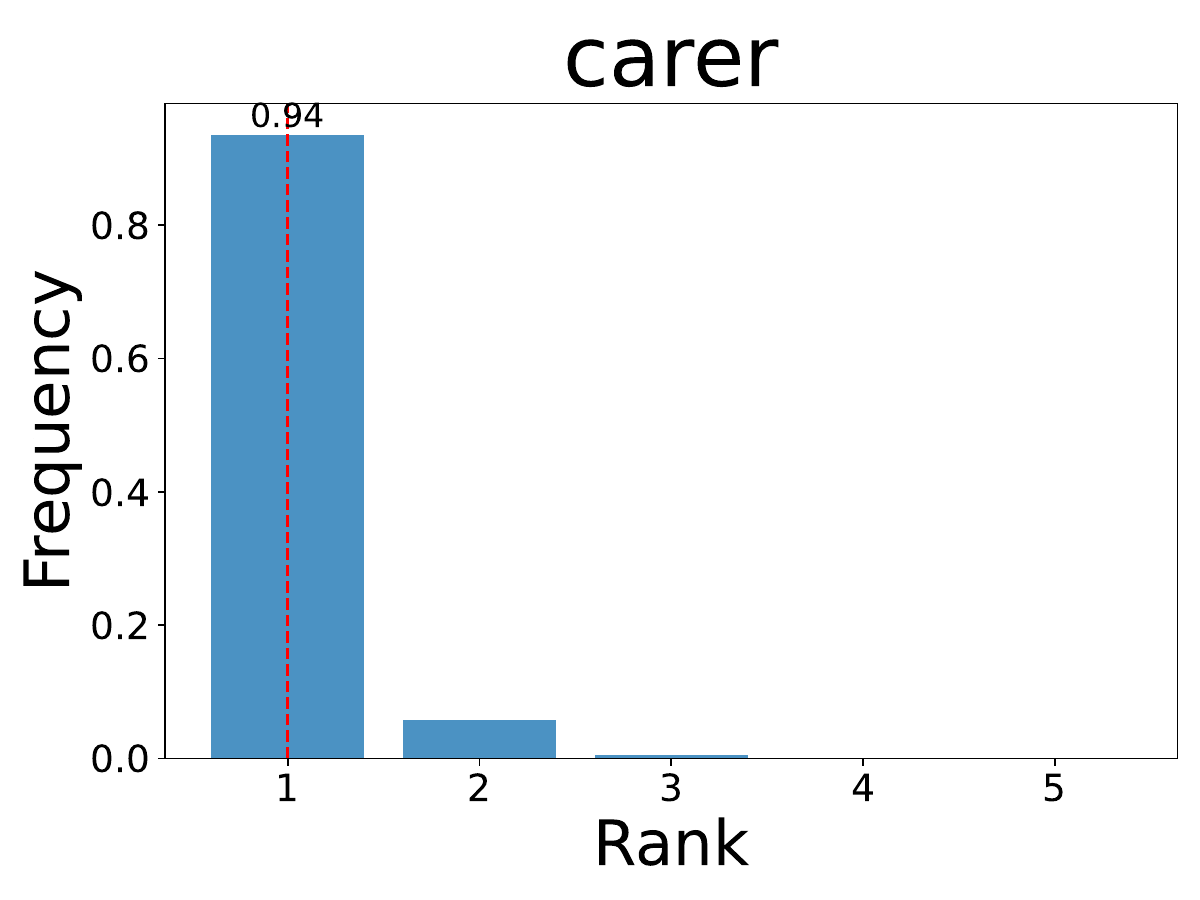}
    \end{subfigure}
    \hspace{0.006\textwidth}
    \begin{subfigure}[b]{0.23\textwidth}
        \centering
        \includegraphics[width=\textwidth]{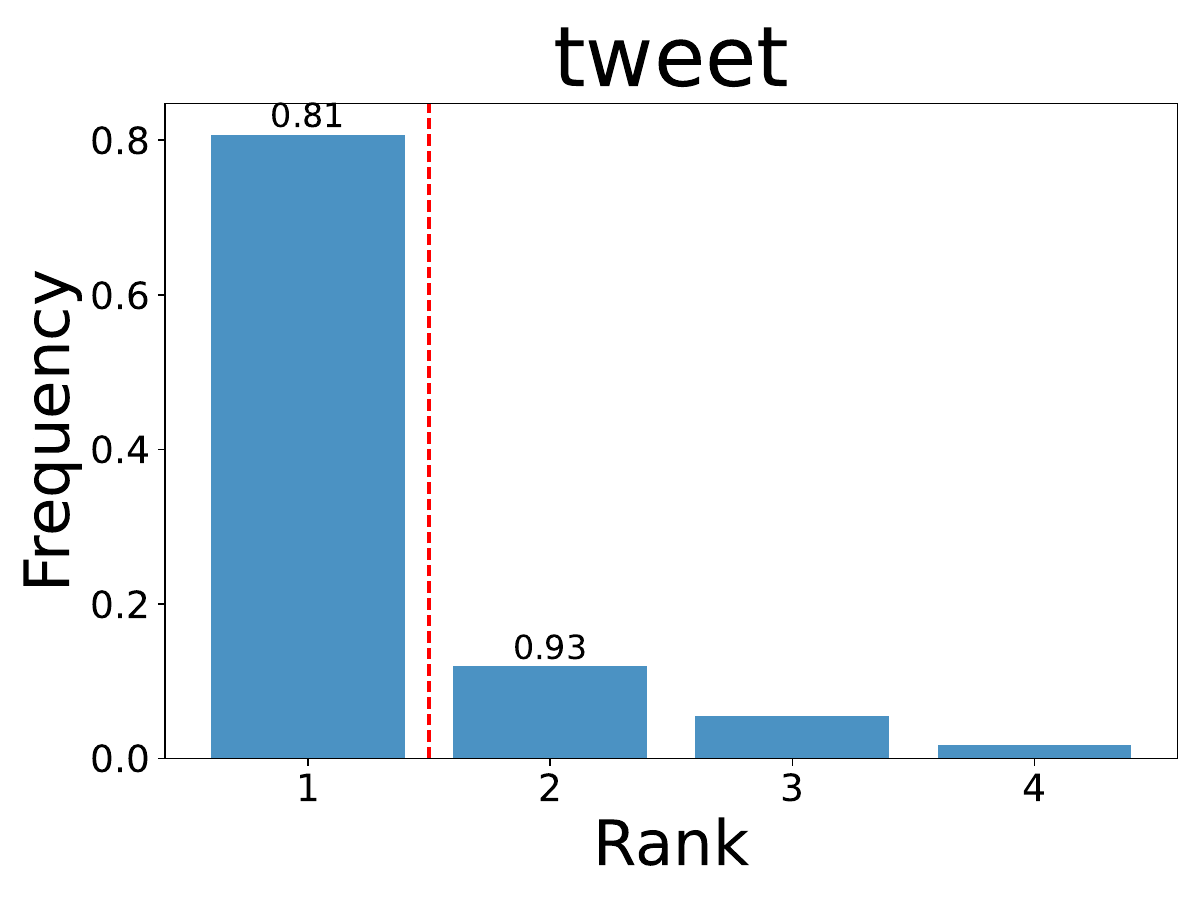}
    \end{subfigure}
    
    \caption{Rank distribution plots of the true class ranks for different datasets. The vertical red line indicates the rank threshold where the cumulative probability exceeds 0.90, corresponding to $r^*_\alpha$ from Algorithm~\ref{alg:rank}. This value aligns with the average prediction set size for $\alpha=0.1$ in Table \ref{tab:alpha_0.1_results}. For CIFAR-100 (cifar100) and 20 Newsgroup (news20), we plot only the rank distribution for the top 10 ranks.}
    \label{fig:rank_figures}
\end{figure*}

\noindent
\textbf{Image Classification: }
We evaluate the performance of prediction sets generated by various methods on four benchmark datasets for image classification: MNIST \cite{lecun2010mnist}, Fashion-MNIST \cite{fmnist}, CIFAR-10 \cite{alex2009learning}, and CIFAR-100 \cite{alex2009learning}. Both MNIST and Fashion-MNIST contain 60000 training images and 10000 test images, while both CIFAR-10 and CIFAR-100 contain 50000 training images and 10000 test images. 
For MNIST and Fashion-MNIST, a Multilayer Perceptron with two hidden layers is used as the base model, whereas for CIFAR-10 and CIFAR-100, a ResNet56 network is employed.

For NLP, the evaluation is conducted on three tasks: Multi-Choice Question-Answering, Topic Classification, and Emotion Recognition. 

\noindent
\textbf{Multi-Choice Question-Answering: }
We evaluate our method on the MMLU benchmark \cite{hendryckstest2021} for the Multi-Choice Question-Answering task, following the approach in \cite{kumar2023conformal}. The datasets are generated using the LLaMA-13B model \cite{touvron2023llama} and consist of questions from three domains: college medicine (191 questions), marketing (269 questions), and public relations (123 questions).

\noindent
\textbf{Topic Classification: }
For topic classification, we assess our method on two datasets: AG News and 20 Newsgroups. AG News is a subset of AG's corpus of news articles, created by combining the titles and description fields of articles from the four largest topic classes: ``World'', ``Sports'', ``Business'', and ``Sci/Tech''. The AG News dataset contains 30000 training samples and 1900 test samples per class. The 20 Newsgroups dataset comprises newsgroup posts on 20 topics, split into a training set of 11314 posts and a test set of 7532 posts.

\noindent
\textbf{Emotion Recognition: }
To evaluate our method's performance on emotion recognition, we utilize two datasets: CARER and TweetEval. CARER consists of English Twitter messages labeled with six basic emotions: anger, fear, joy, love, sadness, and surprise. The dataset has a training set of 15969 tweets and a test set of 2000 tweets. TweetEval, on the other hand, contains tweets categorized into four emotions: anger, joy, optimism, and sadness. This dataset includes 6838 training tweets and a test set of 1421 tweets.

\begin{table*}[!htb]
\centering
\resizebox{\textwidth}{!}{
\fontsize{9}{10}\selectfont
\begin{tabular}{|l|cccc|cccc|cccc|}
\hline
\multirow{2}{*}{Data} & \multicolumn{4}{c|}{Coverage} & \multicolumn{4}{c|}{Size} & \multicolumn{4}{c|}{SSCV} \\
& Ours & APS & RAPS & SAPS & Ours & APS & RAPS & SAPS & Ours & APS & RAPS & SAPS \\
\hline
mnist & 0.90118 & 0.89958 & 0.90120 & 0.89979 & \textbf{0.90130} & 0.92360 & 0.92361 & 0.92799 & \textbf{0.00478} & 0.08604 & 0.03890 & 0.09580 \\
fmnist & 0.90046 & 0.89984 & 0.89974 & 0.89968 & \textbf{0.96771} & 1.12594 & 1.10348 & 1.15082 & \textbf{0.00456} & 0.04865 & 0.00914 & 0.05406 \\
cifar10 & 0.90046 & 0.90025 & 0.90044 & 0.90053 & \textbf{0.92664} & 1.00705 & 0.99371 & 1.00844 & \textbf{0.00460} & 0.05223 & 0.02110 & 0.05191 \\
cifar100 & 0.90016 & 0.89860 & 0.90000 & 0.90027 & 2.67533 & 11.93725 & 2.44703 & \textbf{2.18324} & \textbf{0.05052} & 0.08719 & 0.21440 & 0.13381 \\
marketing & 0.8962 & 0.9016 & 0.8997 & 0.8987 & \textbf{2.6704} & 2.8059 & 2.7920 & 2.7065 & \textbf{0.0368} & 0.0421 & 0.0408 & 0.0699 \\
medicine & 0.9043 & 0.8993 & 0.9000 & 0.9064 & \textbf{3.3550} & 3.3928 & 3.3841 & 3.3723 & \textbf{0.1005} & 0.1023 & 0.1032 & 0.1021 \\
relations & 0.8915 & 0.9048 & 0.8995 & 0.9024 & \textbf{3.2260} & 3.3569 & 3.3268 & 3.2903 & \textbf{0.1007} & 0.1026 & 0.1011 & 0.1033 \\
agnews & 0.9001 & 0.8993 & 0.8997 & 0.8992 & \textbf{0.9703} & 1.1654 & 1.1323 & 1.1807 & \textbf{0.0048} & 0.1000 & 0.0464 & 0.1000 \\
news20 & 0.8992 & 0.9008 & 0.9004 & 0.9004 & 4.0828 & \textbf{3.2626} & 3.9847 & 3.4209 & \textbf{0.0093} & 0.0339 & 0.0403 & 0.0894 \\
carer & 0.8985 & 0.8990 & 0.9015 & 0.9005 & \textbf{0.9305} & 1.0390 & 1.0385 & 1.1065 & \textbf{0.0110} & 0.1000 & 0.0830 & 0.1000 \\
tweet & 0.9013 & 0.9007 & 0.8971 & 0.9007 & \textbf{1.3382} & 1.4801 & 1.4290 & 1.4476 & 0.0584 & 0.1000 & \textbf{0.0414} & 0.1000 \\
\hline
\end{tabular}}
\caption{Evaluation metrics with $\alpha = 0.1$. Coverage (\ref{eq:coverage}): greater than or closer to $1-\alpha=0.9$ is better. Size (\ref{eq:size}): smaller is better. SSCV (\ref{eq:sscv}): smaller is better. \textbf{Bold} numbers indicate optimal performance.}
\label{tab:alpha_0.1_results}
\end{table*}

Let us denote the test set by $\mathcal I_3$. We assess the performance of the different methods using the following three metrics.

\textbf{Coverage Rate (Coverage):}
The coverage rate measures the proportion of test instances where the true label is included in the prediction set. A higher coverage rate indicates better performance.
\begin{equation}\label{eq:coverage}
    \textrm{Coverage} = \frac{1}{|\mathcal{I}_{3}|} \sum_{i\in \mathcal{I}_{3}} \mathbbm{1}(y_i \in \hat{C}(x_i)),
\end{equation}

\textbf{Average Size (Size):}
The average size refers to the mean number of labels in the prediction sets. Smaller sizes are consider more precise and informative of the labels in the prediction set.
\begin{equation}\label{eq:size}
    \textrm{Size} = \frac{1}{|\mathcal{I}_{3}|} \sum_{i\in \mathcal{I}_{3}}  |\hat{C}(x_i)|.
\end{equation}

\textbf{Size-Stratified Coverage Violation (SSCV):}
The size-stratified coverage violation \cite{angelopoulos2021uncertainty} evaluates the consistency of coverage across different prediction set sizes $\{S_j\}_{j=1}^s$, where $S_1, S_2, \dots, S_s$ are partitions of $[K]$. Let $J_j = \{i\in\mathcal{I}_3: |\hat{C}(x_i)| \in S_j\}$ denote the indices of examples stratified by the prediction set size $S_j$. Then we define
\begin{equation}\label{eq:sscv}
 \mathrm{SSCV}(\hat{C}, \{S_j\}_{j=1}^s)
 = \sup_{j\in[s]} \left|\frac{|\{i \in J_j: y_i \in \hat{C}(x_i)\}|}{|J_j|} - (1 - \alpha)\right|.
\end{equation}
Smaller SSCV indicates more stable coverage.

Throughout the experiments, the split-conformal prediction framework is employed to construct the prediction sets. Different $\alpha$ values ranging from 0.1 to 0.3 are chosen, and the mean Coverage, Size, and SSCV metrics are computed across 100 repetitions. %

\begin{figure*}[!tb]
\centering
\includegraphics[width=0.99\textwidth]{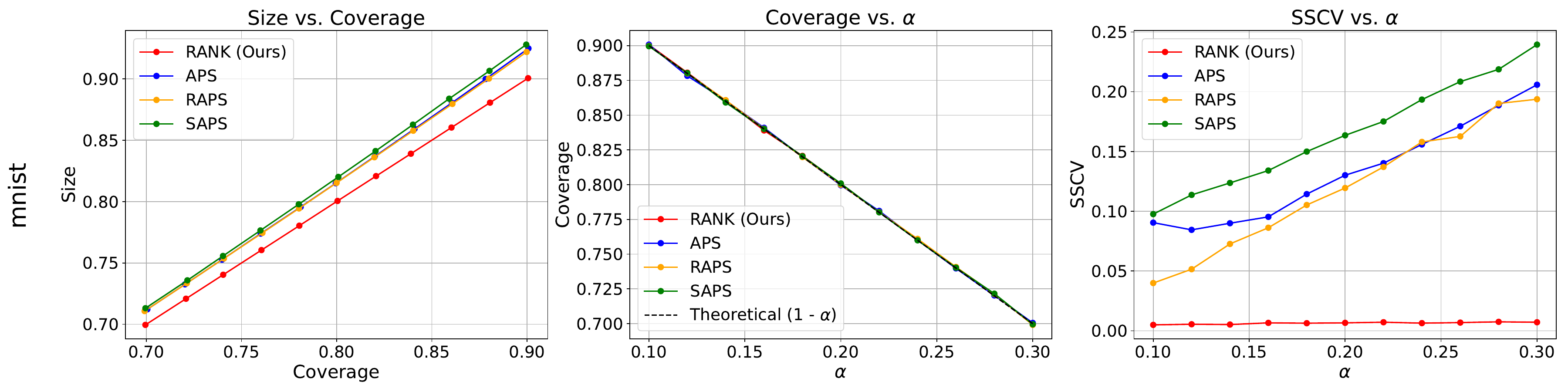}
\includegraphics[width=0.99\textwidth]{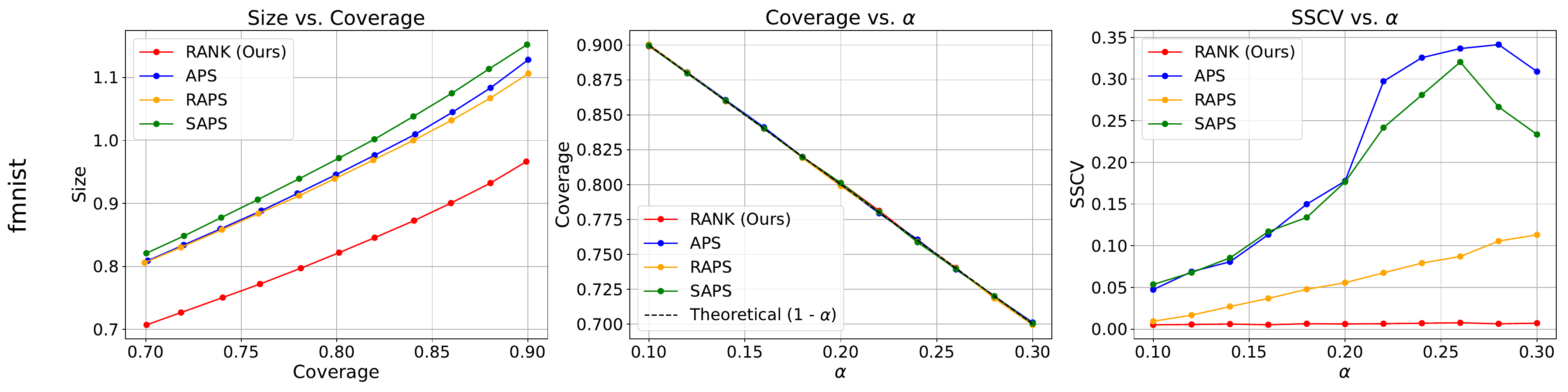}
\includegraphics[width=0.99\textwidth]{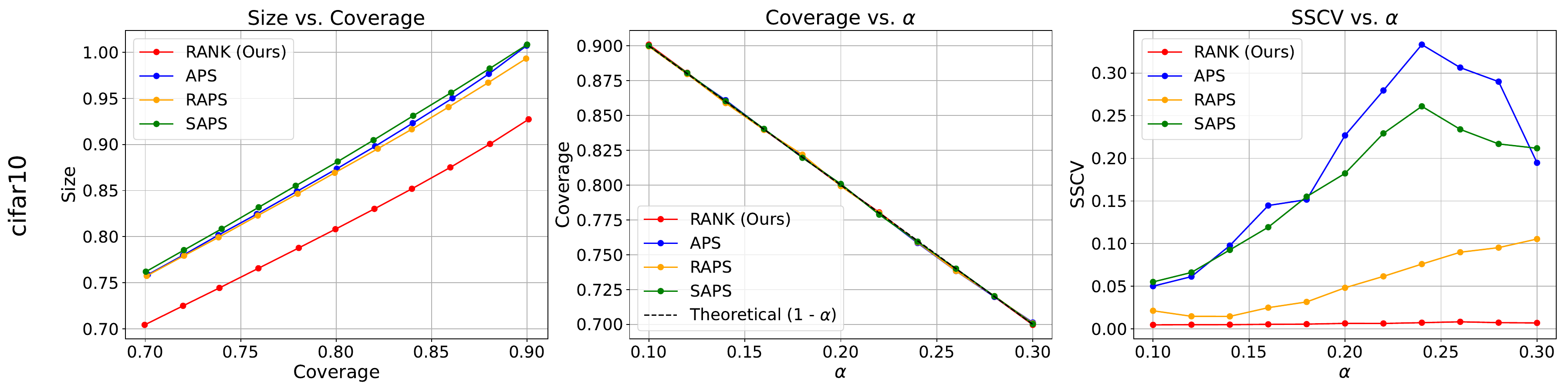}
\includegraphics[width=0.99\textwidth]{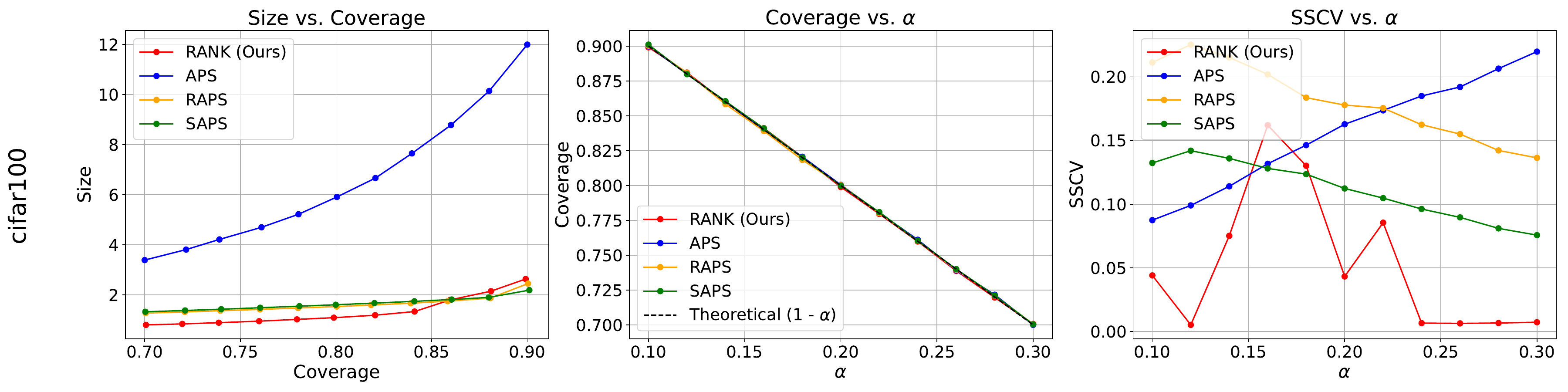}
\caption{Results for Image Classification on MNIST (mnist), Fashion-MNIST (fmnist), CIFAR-10 (cifar10), and CIFAR-100 (cifar100).}\label{fig: image classification}
\end{figure*}

Table \ref{tab:alpha_0.1_results} presents the results for $\alpha=0.1$, while Figures \ref{fig: image classification}, \ref{fig: mcqa}, and \ref{fig: topic emotion} show the results on various tasks for $\alpha$ ranging from $0.1$ to $0.3$. The left subfigure compares the Size vs. Coverage trade-off. A lower curve indicates that the method can achieve the desired coverage using a smaller prediction set size. The middle subfigure illustrates the relationship between coverage and $\alpha$. Methods closer to the theoretical line $1-\alpha$ are considered better. The right subfigure displays the SSCV vs. $\alpha$ plot. A lower curve means that the method can achieve $1-\alpha$ coverage more consistently across different strata of prediction set sizes. 

The experiment results show that our method (RANK) preforms better than other completing methods in most classification tasks, when measuring the performance by prediction set size versus coverage, except for the datasets cifar100 and news20 with $\alpha\le 0.15$. The underperformance of our proposed method in these cases can be attributed to the large number of classes in these datasets. This results in a more dispersed rank distribution, making it challenging for our method to construct efficient prediction sets, particularly when $\alpha$ is small.
To investigate our method's suboptimal performance on the cifar100 and news20 dataset compared to other methods, we examine their rank distribution plots in Figure~\ref{fig:rank_figures}. The plots reveal a long-tailed distribution, with many instances having high ranks. Our method may be less effective in such cases, as it does not explicitly minimize the tail probability of the ranks, unlike the APS-type approach, which is designed to handle these situations more effectively.

In other datasets, our performance is overwhelmingly better than the others. When comparing the SSCV metric with other methods, our approach demonstrates remarkably consistent coverage across most datasets.

\begin{figure*}[!tb]
\centering
\includegraphics[width=0.99\textwidth]{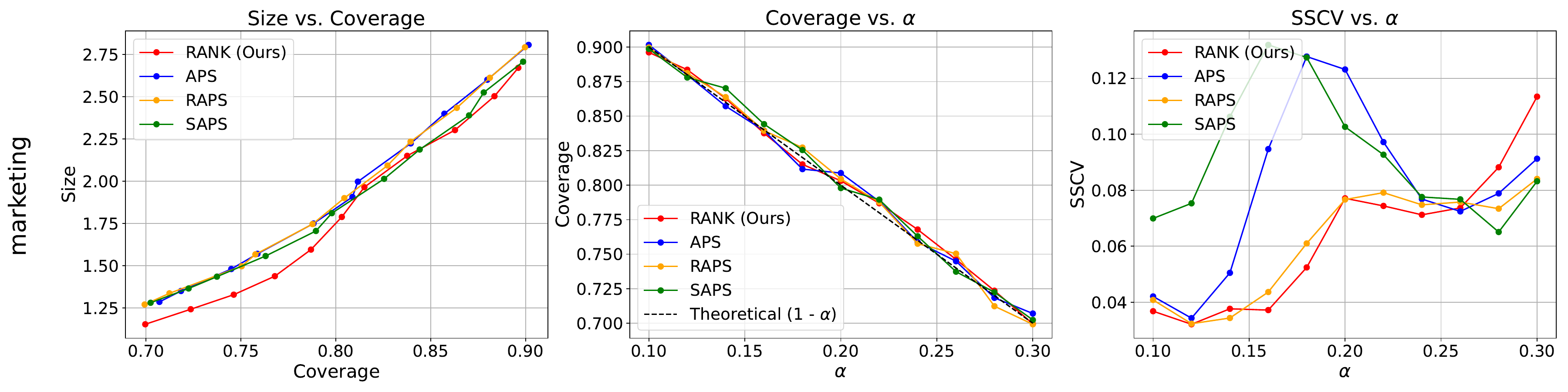}
\includegraphics[width=0.99\textwidth]{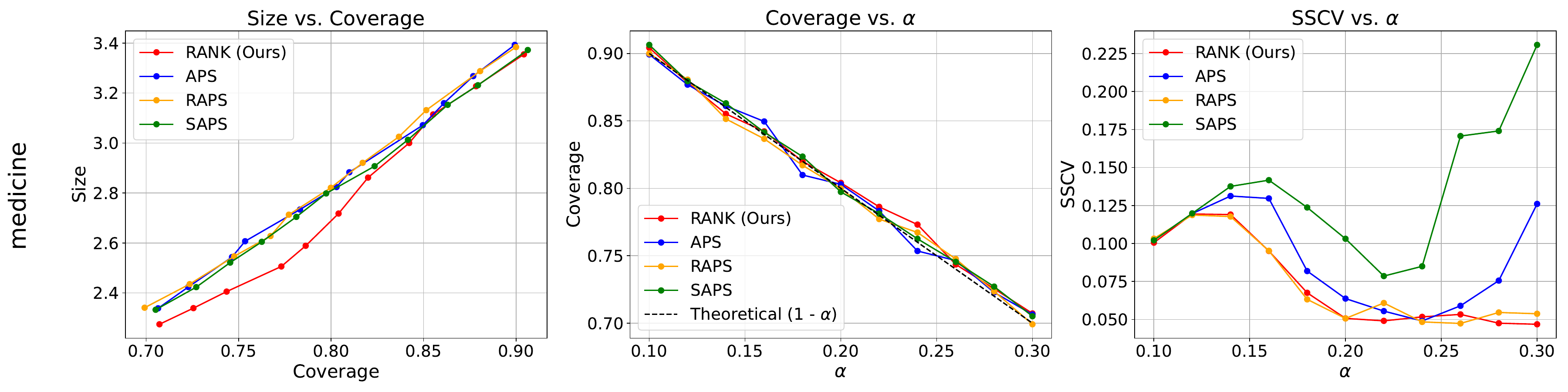}
\includegraphics[width=0.99\textwidth]{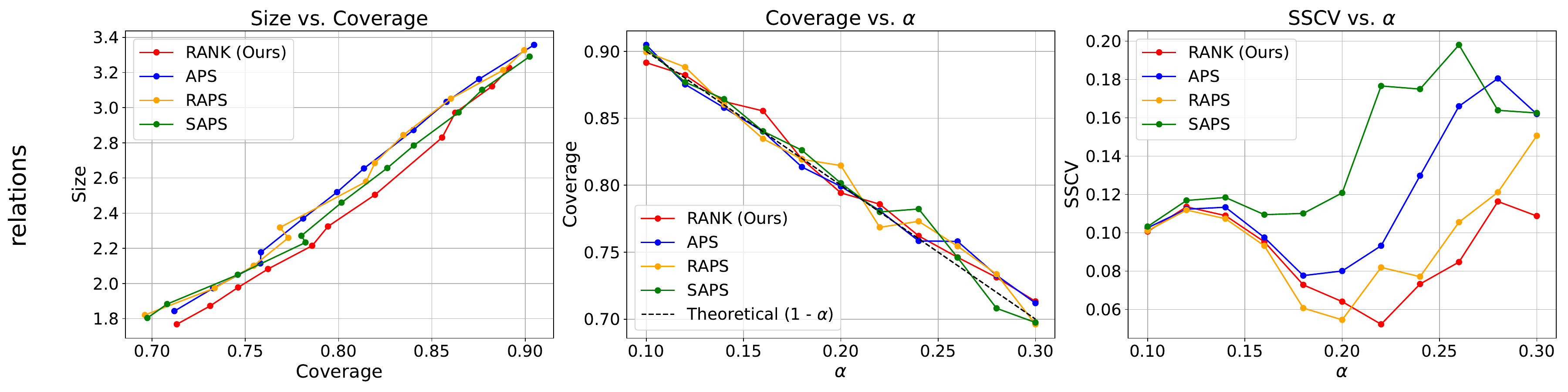}
\caption{Results for the Multi-Choice Question-Answering datasets.}\label{fig: mcqa}
\end{figure*}

\begin{figure*}[!tb]
\centering
\includegraphics[width=0.99\textwidth]{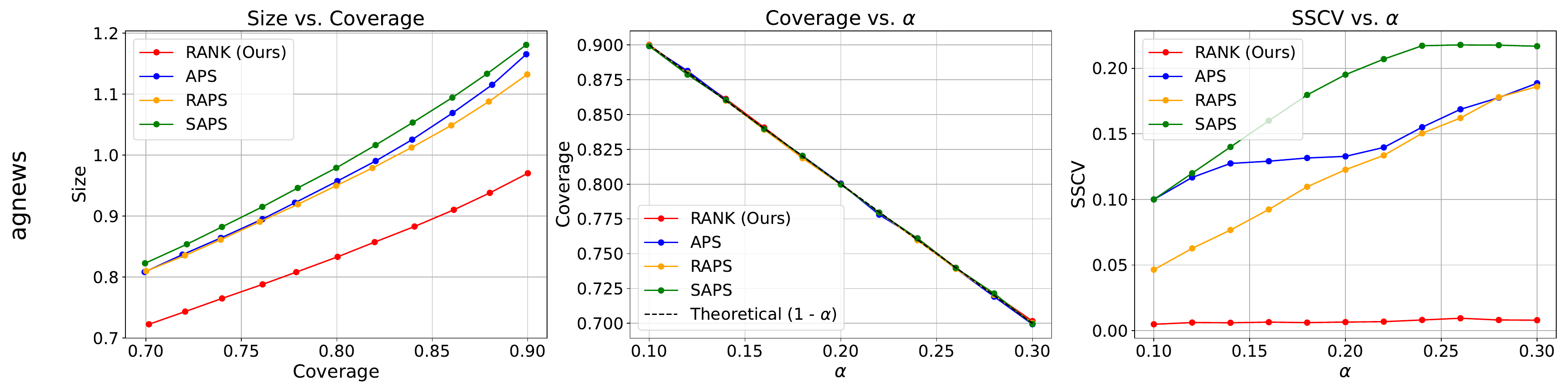}
\includegraphics[width=0.99\textwidth]{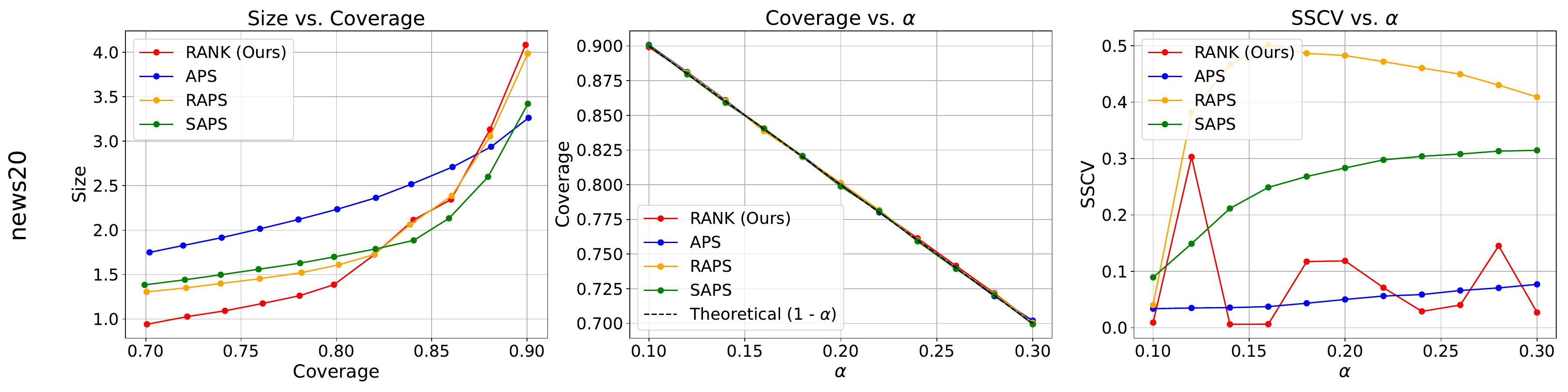}
\includegraphics[width=0.99\textwidth]{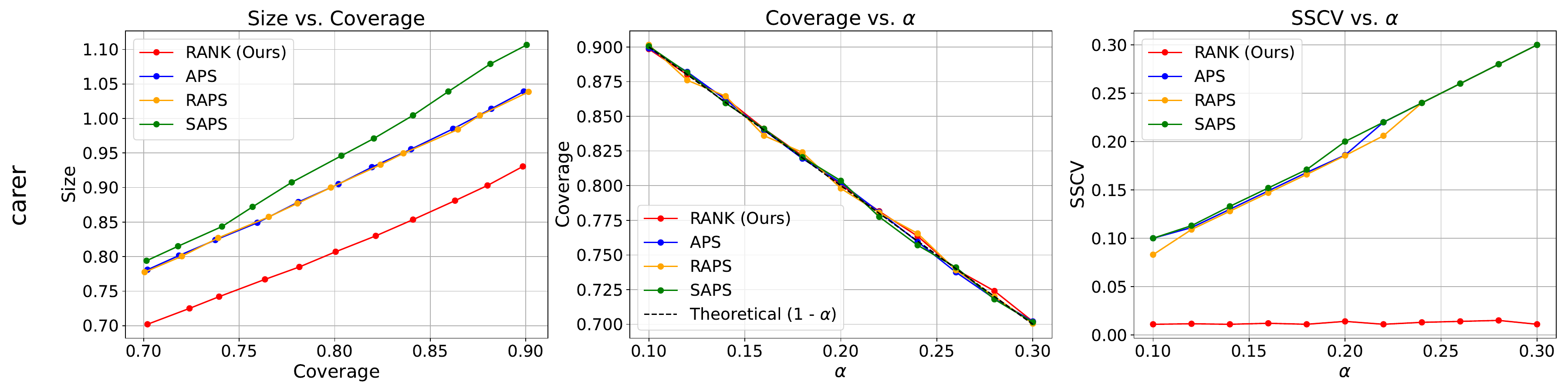}
\includegraphics[width=0.99\textwidth]{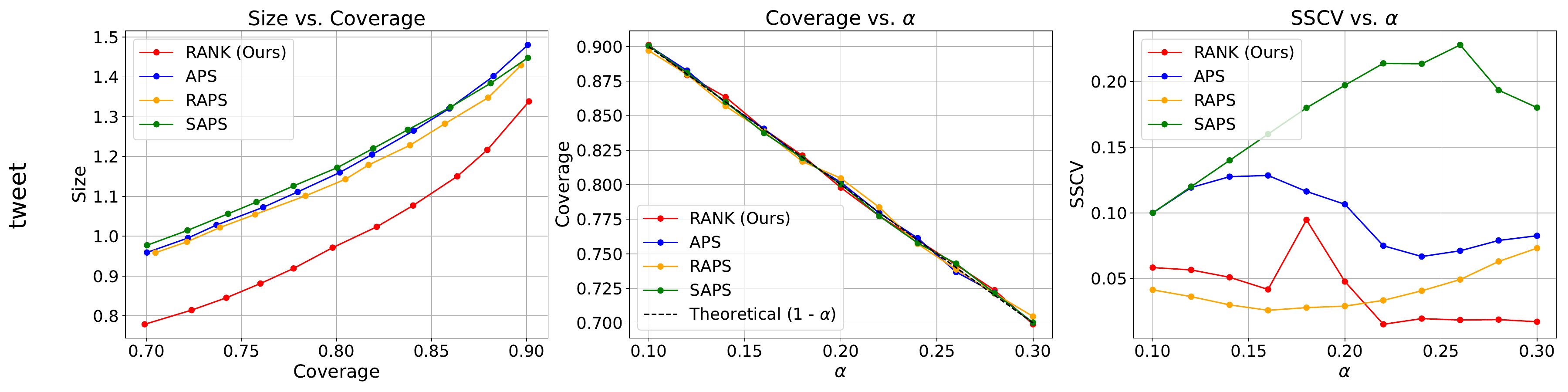}
\caption{Results for Topic Classification on AG News (agnews), 20 Newsgroups (news20); and Emotion Recognition on CARER (carer), TweetEval (tweet).}\label{fig: topic emotion}
\end{figure*}

\section{Discussion and Future Work}

Our proposed conformal prediction method for general machine learning classification tasks has demonstrated promising results in achieving high performance. The key contribution of our work is the proposal of a novel rank-based score function that is suitable for classification models that are not necessarily well-calibrated but predict the orders rather correctly. Moreover, we have shown that the expected size of the produced conformal prediction set can be analyzed theoretically based on the rank distribution of the classification model applied to a specific dataset. However, there are several areas where further research and improvements can be made.

\noindent
\textbf{Performance on Large Number of Classes.} Our approach's performance may deteriorate when faced with a very large number of classes, as seen in some datasets when $\alpha\le 0.15$. A potential future direction could involve combining our idea of minimizing the prediction set size with the concept of minimizing the tail probability, similar to APS-type methods, to improve results in these scenarios.

\noindent
\textbf{Multi-Label Classification.} Extending our approach to handle multi-label classification, where each instance can be assigned multiple labels simultaneously, is an important future research direction. While most single-label methods can be directly applied to multi-label scenarios, accounting for label dependence becomes challenging. Using statistical methods to estimate label co-occurrence could help capture complex label relationships and improve classification performance.

\section{Conclusion}

We proposed a novel conformal prediction method, RANK, for general classification tasks that effectively leverages confidence scores from classification models to construct prediction sets with desired coverage and minimal size. The key contribution of our work is the proposal of a rank-based score function that is suitable for classification models that are not necessarily well-calibrated but predict the orders rather correctly. We have also shown that the expected size of the produced conformal prediction set can be analyzed theoretically based on the rank distribution of the classification model applied to a specific dataset. Rigorous theoretical analysis and extensive experiments demonstrate our method's clear superiority over existing techniques, providing reliable uncertainty quantification for classifiers. Our significant contributions advance the field of uncertainty quantification in machine learning and enable the development of more reliable classification technologies.

\bibliographystyle{acm}

\end{document}